\documentclass[conference]{IEEEtran}
\IEEEoverridecommandlockouts
\usepackage{cite}
\usepackage{textcomp}

\usepackage[american]{babel}

\usepackage{booktabs} 
\usepackage{tikz} 

\usepackage{url}

\usepackage{graphicx}
\usepackage{amsmath,amssymb,amsfonts,amstext,mathrsfs,xcolor,amsthm}
\usepackage[utf8]{inputenc} 
\usepackage[T1]{fontenc}    
\usepackage{hyperref}       
\hypersetup{
    colorlinks = true,
    linkbordercolor = {blue},
    citecolor = {blue}
}
\usepackage{booktabs}       
\usepackage{nicefrac}       
\usepackage{microtype}      
\usepackage{cleveref}
\usepackage{dsfont}
\usepackage{enumitem}
\usepackage{thm-restate}
\usepackage{color}
\usepackage{algorithm}
\usepackage[algo2e]{algorithm2e}
\usepackage{bbm}
\usepackage{makecell}

\usepackage[normalem]{ulem}


\newcommand{\prox}[1]{\mathrm{prox}_{#1}}

\newcommand{\zero}{\mathbf{0}}

\newcommand{\argmin}{\mathop{\mathrm{argmin}}}

\newcommand{\inner}[2]{\langle #1, #2 \rangle}
\newcommand{\red}[1]{\color{red} #1 \color{black}}

\newtheorem{assum}{Assumption}

\allowdisplaybreaks[4]   

\def\BibTeX{{\rm B\kern-.05em{\sc i\kern-.025em b}\kern-.08em
    T\kern-.1667em\lower.7ex\hbox{E}\kern-.125emX}}

\begin{document}

	
\title{Accelerated Proximal Alternating Gradient-Descent-Ascent for Nonconvex Minimax Machine Learning}

\author{\IEEEauthorblockN{Ziyi Chen}
\IEEEauthorblockA{\textit{Electrical \& Computer Engineering} \\
\textit{University of Utah}\\
Salt Lake City, US \\
u1276972@utah.edu}
\and
\IEEEauthorblockN{Shaocong Ma}
\IEEEauthorblockA{\textit{Electrical \& Computer Engineering} \\
\textit{University of Utah}\\
Salt Lake City, US \\
s.ma@utah.edu}
\and
\IEEEauthorblockN{Yi Zhou}
\IEEEauthorblockA{\textit{Electrical \& Computer Engineering} \\
\textit{University of Utah}\\
Salt Lake City, US \\
yi.zhou@utah.edu}
}

\maketitle              

\begin{abstract}
Alternating gradient-descent-ascent (AltGDA) is an optimization algorithm that has been widely used for model training in various machine learning applications, which aims to solve a nonconvex minimax optimization problem. However, the existing studies show that it suffers from a high computation complexity in nonconvex minimax optimization.  
In this paper, we develop a single-loop and fast AltGDA-type algorithm that leverages proximal gradient updates and momentum acceleration to solve regularized nonconvex minimax optimization problems. 
By leveraging the momentum acceleration technique, we prove that the algorithm converges to a critical point in nonconvex minimax optimization and achieves a computation complexity in the order of $\mathcal{O}(\kappa^{\frac{11}{6}}\epsilon^{-2})$, where $\epsilon$ is the desired level of accuracy and $\kappa$ is the problem's condition number. {Such a computation complexity improves the state-of-the-art complexities of single-loop GDA and AltGDA algorithms (see the summary of comparison in \Cref{table1})}. We demonstrate the effectiveness of our algorithm via an experiment on adversarial deep learning.

\end{abstract}

\begin{IEEEkeywords}
Minimiax optimization, alternating gradient descent ascent, proximal gradient, momentum, complexity.
\end{IEEEkeywords}

\section{Introduction}
Minimax optimization is an emerging and important optimization framework that covers a variety of modern machine learning applications. Some popular application examples include generative adversarial networks (GANs) \cite{goodfellow2014generative}, adversarial machine learning \cite{Sinha2018CertifyingSD}, game theory \cite{ferreira2012minimax}, reinforcement learning \cite{qiu2020single}, etc. A standard minimax optimization problem is written as $\min_{x\in \mathbb{R}^{m}}\max_{y\in \mathcal{Y}}~ f(x,y)$, where $f$ is a differentiable bivariate function.
A basic algorithm for solving the above minimax optimization problem is the gradient-descent-ascent (GDA), which simultaneously performs gradient descent update and gradient ascent update on the variables $x$ and $y$, respectively, i.e., $x_{t+1}=x_t - \eta_{x}\nabla_1 f(x_t,y_t)$, $y_{t+1}=y_t + \eta_{y}\nabla_2 f({x_t}, y_t)$. Here, $\nabla_1$ and $\nabla_2$ denote the gradient operator with regard to the first and the second variable, respectively.
The convergence rate of GDA has been studied under various types of geometries of the minimax problem, including strongly-convex-strongly-concave geometry \cite{fallah2020optimal}, nonconvex-(strongly)-concave geometry \cite{lin2019gradient} and {L}ojasiewicz-type geometry \cite{chen2021}.
Recently, by leveraging the popular momentum technique \cite{Nesterov2014,Tseng2010,Beck2009,Ghadimi2016b,Li2015,Li2017} for accelerating gradient-based algorithms, accelerated variants of GDA have been proposed for strongly-convex-strongly-concave  \cite{zhang2020suboptimality} and nonconvex-strongly-concave \cite{huang2021efficient} minimax optimization. There are other accelerated GDA-type algorithms that achieve a near-optimal complexity \cite{Lin2020a,zhang2021complexity}, but they involve complex nested-loop structures and require function smoothing with many fine-tuned hyper-parameters, which are not used in practical minimax machine learning.

Another important variant of GDA that has been widely used in practical training of minimax machine learning is the alternating-GDA (AltGDA) algorithm, which updates the two variables $x$ and $y$ alternatively via $x_{t+1}=x_t - \eta_{x}\nabla_1 f(x_t,y_t)$, $y_{t+1}=y_t + \eta_{y}\nabla_2 f({x_{t+1}}, y_t)$. Compared with the update of GDA, the $y$-update of AltGDA uses the fresh $x_{t+1}$ instead of the previous $x_t$, and it is shown to converge faster than the standard GDA algorithm \cite{boct2020alternating,bailey2020finite,gidel2019negative,xu2020unified}. 
Despite the popularity of the AltGDA algorithm, it is shown to suffer from a high computation complexity in nonconvex minimax optimization. 
Therefore, this study aims to improve the complexity of AltGDA by leveraging momentum acceleration techniques. In particular, in the existing literature, the convergence of momentum accelerated AltGDA is only established for convex-concave minimax problems \cite{yadav2018stabilizing} and bilinear minimax problems \cite{gidel2019negative}, and it has not been explored in nonconvex minimax optimization that covers modern machine learning applications.
Therefore, this study aims to fill in this gap by developing a single-loop proximal-AltGDA algorithm with momentum acceleration for solving a class of regularized nonconvex minimax problems, and analyze its convergence and computation complexity. 

\subsection{Our Contribution}
We are interested in a class of regularized nonconvex-strongly-concave minimax optimization problems, where the regularizers are convex functions that can be possibly non-smooth. To solve this class of minimax problems, we propose a single-loop proximal-AltGDA with momentum algorithm (referred to as proximal-AltGDAm). The algorithm takes single-loop updates that consist of a proximal gradient descent update with the heavy-ball momentum, and a proximal gradient ascent update with the Nesterov's momentum. Our algorithm extends the applicability of the conventional momentum acceleration schemes (heavy-ball and Nesterov's momentum) for nonconvex minimization to nonconvex minimax optimization. 


We study the convergence property of Proximal-AltGDAm. Specifically, under standard smoothness assumptions on the objective function and by viewing the accelerated alternating GDA updates as inexact accelerated gradient updates, we develop an analysis to show that every limit point of the parameter sequences generated by the algorithm is a critical point of the nonconvex regularized minimax problem. Moreover, to achieve an $\epsilon$-accurate critical point, the overall computation complexity (i.e., number of gradient {and proximal mapping} evaluations) is of the order $\mathcal{O}\big(\kappa^{\frac{11}{6}}\epsilon^{-2}\big)$, where $\kappa>1$ is the condition number of the problem. Thanks to momentum acceleration and a tight analysis in our technical proof, such a computation complexity is lower than that of the existing single-loop GDA-type algorithms. See Table \ref{table1} for a summary of comparison of the computation complexities and Appendix \ref{supp: table} for their derivation. 





\begin{table}[H]
\caption{Comparison of computation complexity (number of gradient {and proximal mappings} evaluations) of the existing single-loop GDA-type algorithms in nonconvex-strongly-concave minimax optimization,  where $\kappa\ge 1$ is the problem condition number.}
\begin{center}
\begin{small}
\begin{tabular}{ccccc}
\hline
 & Alternating  & Momentum & Computation \\ 
 & update & acceleration & complexity \\ \hline
(Chen, et.al) \cite{chen2021} & $\times$ & $\times$ & $\mathcal{O}(\kappa^6\epsilon^{-2})$ \\ \hline
(Huang, et.al) \cite{huang2021efficient} & $\times$ & \checkmark & $\mathcal{O}(\kappa^3\epsilon^{-2})$ \\ \hline
(Lin, et.al) \cite{lin2019gradient} & $\times$ & $\times$ & $\mathcal{O}(\kappa^2\epsilon^{-2})$ \\ \hline
(Xu, et.al) \cite{xu2020unified} & \checkmark & $\times$ & $\mathcal{O}(\kappa^5\epsilon^{-2})$ \\ \hline
(Bo{\c{t}} and B{\"o}hm) \cite{boct2020alternating} & \checkmark & $\times$ & $\mathcal{O}(\kappa^2\epsilon^{-2})$ \\ \hline
\textbf{\red{This work}} & \checkmark  & \checkmark  & $\mathcal{O}(\kappa^{\frac{11}{6}}\epsilon^{-2})$ \\ \hline
\end{tabular} \label{table1}
\end{small}
\end{center}
\end{table}

\subsection{Other Related Work}\label{sec: related_work}
\noindent\textbf{GDA-type algorithms: } 
\cite{nedic2009subgradient} studied a slight variant of GDA by replacing gradients with subgradients for convex-concave non-smooth minimax optimization. \cite{xu2020unified} studied AltGDA which applies $\ell_2$ regularizer to the local objective function of GDA followed by projection onto the constraint sets {and obtained its convergence rate under nonconvex-concave and convex-nonconcave geometry}. \cite{mokhtari2020unified} also studied two variants of GDA, Extra-gradient and Optimistic GDA, and obtained linear convergence rate under bilinear geometry and strongly-convex-strongly-concave geometry. \cite{cherukuri2017saddle,jin2019local} studied GDA in continuous time dynamics using differential equations. \cite{adolphs2019local} analyzed a second-order variant of the GDA algorithm. 

Many other studies have developed stochastic GDA-type algorithms. \cite{lin2019gradient,yang2020global} analyzed stochastic GDA and stochastic AltGDA respectively. Variance reduction techniques have been applied to stochastic minimax optimization, including SVRG-based \cite{du2019linear,yang2020global}, SARAH/SPIDER-based \cite{xu2020enhanced,luo2020stochastic}, STORM \cite{qiu2020single} and its gradient free version \cite{huang2020accelerated}. 

\noindent \textbf{GDA-type algorithms with momentum:} For strongly-convex-strongly-concave problems, \cite{zhang2020suboptimality} studied accelerated GDA with negative momentum. \cite{Wang2020a,Lin2020a} developed nested-loop AltGDA algorithms with Nesterov's Accelerated Gradient Descent that achieve improved complexities. For convex-concave problems, 
\cite{daskalakis2018limit} analyzed stable points of both GDA {and optimistic GDA that apply negative momentum for acceleration}. Moreover, for nonconvex-(strongly)-concave problems, {\cite{wang2019solving} developed a single-loop GDA with momentum and Hessian preconditioning and studied its convergence to a local minimax point}. {\cite{huang2021efficient} developed a mirror descent ascent algorithm with momentum which includes GDAm as a special case.} \cite{nouiehed2019solving} studied nested-loop GDA where multiple gradient ascent steps are performed, and they also studied the momentum-accelerated version. 
Similarly, \cite{qiu2020single,huang2020accelerated} developed GDA with momentum scheme and STORM variance reduction, and a similar version of this algorithm is extended to minimax optimization on Riemann manifold \cite{huang2020gradient}. \cite{guo2021stochastic} developed a single-loop Primal-Dual Stochastic Momentum algorithm with ADAM-type methods.


\section{Problem Formulation and Preliminaries}\label{sec: pre}

In this section, we introduce the problem formulation and technical assumptions. We consider the following class of regularized minimax optimization problems.
\begin{align}
\min_{x\in \mathbb{R}^m}\max_{y\in \mathcal{Y}}~ f(x,y)+g(x)-h(y), \tag{P}
\end{align}
where $f: \mathbb{R}^m\times \mathcal{Y}\to \mathbb{R}$ is differentiable and nonconvex-strongly-concave, $g,h$ are possibly non-smooth convex regularizers, and $\mathcal{Y}$ is a bounded set with diameter $R$. In particular, define the envelope function $\Phi(x):=\max_{y\in\mathcal{Y}} f(x,y)-h(y)$, then the problem $\mathrm{(P)}$ can be rewritten as the minimization problem $\min_{x\in \mathbb{R}^m} \Phi(x)+g(x)$. 

Throughout the paper, we adopt the following assumptions on the problem (P). These are standard assumptions that have been considered in the existing literature \cite{lin2019gradient,chen2021}.
\begin{assum}\label{assum: f}
	The minimax problem $\mathrm{(P)}$ satisfies:
	\begin{enumerate}[leftmargin=*,topsep=0pt,noitemsep]
		\item Function $f(\cdot,\cdot)$ is $L$-smooth and function $f(x,\cdot)$ is $\mu$-strongly concave for all $x$;
		\item Function $(\Phi+g)(x)$ is bounded below and has compact sub-level sets;
		\item Function $g,h$ are proper and convex.
	\end{enumerate}
\end{assum} 

In particular, item 3 of the above assumption allows the regularizers $g,h$ to be any convex function that can be possibly non-smooth.
Next, consider the mapping $y^*(x)=\arg\max _{y\in\mathcal{Y}}f(x,y)-h(y)$, which is  uniquely defined due to the strong concavity of $f(x,\cdot)$. The following proposition is proved in \cite{lin2019gradient,chen2021} that characterizes the Lipschitz continuity of the mapping $y^*(x)$ and the smoothness of the function $\Phi(x)$. Throughout the paper, we denote $\kappa=L/\mu>1$ as the condition number, and denote $\nabla_1 f(x,y), \nabla_2 f(x,y)$ as the gradients with respect to the first and the second input variables, respectively.


\begin{restatable}[\!\!\cite{lin2019gradient,chen2021}]{proposition}{propphy}\label{prop_Phiystar}
	Let \Cref{assum: f} hold. Then, the mapping $y^*(x)$  is $\kappa$-Lipschitz continuous and
	 the function $\Phi(x)$ is $L(1+\kappa)$-smooth with $\nabla\Phi(x)=\nabla_1 f(x,y^*(x))$.
\end{restatable}

Lastly, recall that  the minimax problem (P) is equivalent to the minimization problem $\min_{x\in \mathbb{R}^m} \Phi(x)+g(x)$. Therefore, the optimization goal of the minimax problem (P) is to {find a critical point $x^*$ of the nonconvex function $\Phi(x)+g(x)$} that satisfies the optimality condition $\zero \in \partial (\Phi + g)(x^*)$, where $\partial$ denotes the subdifferential operator.




\section{Proximal Alternating GDA with Momentum}
In this section, we propose a {single-loop proximal alternating-GDA with momentum (proximal-AltGDAm)} algorithm to solve the regularized minimax problem (P). 

We first recall the update rules of the basic proximal {alternating-GDA (proximal-AltGDA)} algorithm \cite{boct2020alternating} for solving the problem (P). Specifically, {proximal-AltGDA} alternates between the following two proximal-gradient updates (a.k.a. forward-backward splitting updates \cite{lions1979splitting}).

\textbf{(Proximal-AltGDA)}:
\begin{equation}
\left\{
\begin{aligned}
x_{t+1} &= \prox{\eta_x g}\big({x}_t - \eta_{x}\nabla_1 f(x_t,y_t)\big),\\
y_{t+1} &= \prox{\eta_y h}\big({y}_t + \eta_{y}\nabla_2 f({x_{t+1}}, {y}_t)\big). \nonumber
\end{aligned}
\right.
\end{equation}
To elaborate, the first update is a proximal gradient descent update that aims to minimize the nonconvex function $f(x, y_t) + g(x)$ {from} the current point $x_t$, and the second update is a proximal gradient ascent update that aims to maximize the strongly-concave function $f({x_{t+1}}, y) - h(y)$ {from} the current point $y_t$.  
More specifically, the two proximal gradient mappings are formally defined as 
\begin{align*}
&\prox{\eta_x g}\big(x_t - \eta_{x}\nabla_1 f(x_t,y_t)\big) \\
&:= \argmin_{u \in \mathbb{R}^m} \Big\{g(u) + \frac{1}{2\eta_{x}} \| u - x_t +\eta_{x} \nabla_1 f(x_t,y_t)\|^2 \Big\} , \\
&\prox{\eta_y h}\big(y_t + \eta_{y}\nabla_2 f({x_{t+1}},y_t)\big) \\
&:= \argmin_{v \in \mathcal{Y}} \Big\{h(v) + \frac{1}{2\eta_{y}} \| v - y_t -\eta_{y} \nabla_2 f({x_{t+1}},y_t)\|^2 \Big\}.
\end{align*}

Compared with the standard (proximal) GDA algorithm \cite{lin2019gradient,chen2021}, the proximal ascent step of proximal-AltGDA evaluates the gradient at the freshest point $x_{t+1}$ instead of $x_t$. Such an alternative update is widely used in practice and usually leads to better convergence properties \cite{boct2020alternating,bailey2020finite,gidel2019negative,xu2020unified}.

Next, we introduce momentum schemes to accelerate the convergence of proximal-AltGDA.
As the two proximal gradient steps of proximal-AltGDA are used to solve two different types of optimization problems, namely, the nonconvex problem $f(x,y_t)+g(x)$ and the strongly-concave problem $f({x_{t+1}}, y) - h(y)$, we consider applying different momentum schemes to accelerate these proximal gradient updates. Specifically, the proximal gradient descent step minimizes a composite nonconvex function, and we apply the heavy-ball momentum scheme \cite{POLYAK1964} that was originally designed for accelerating nonconvex optimization. Therefore, we propose the following proximal gradient descent with heavy-ball momentum update for minimizing the nonconvex part of the problem (P).

\textbf{(Heavy-ball momentum):}
\begin{equation}
\left\{
\begin{aligned}
	\widetilde{x}_t &= x_t + \beta (x_t - x_{t-1}),\\
x_{t+1} &= \prox{\eta_x g}\big(\widetilde{x}_t - \eta_{x}\nabla_1 f(x_t,y_t)\big). \nonumber
\end{aligned}
\right.
\end{equation}
To explain, the first step is an extrapolation step that applies the momentum term $\beta (x_t - x_{t-1})$ (with momentum coefficient $\beta>0$), and the second proximal gradient step updates the extrapolation point $\widetilde{x}_t$ using the original gradient $\nabla_1 f(x_t,y_t)$. In conventional gradient-based optimization, gradient descent with such a heavy-ball momentum is guaranteed to find a critical point of smooth nonconvex functions \cite{Ochs2014,Ochs2018}. 

On the other hand, as the proximal gradient ascent step of proximal-AltGDA maximizes a composite strongly-concave function, we are motivated to apply the popular Nesterov's momentum scheme \cite{nesterov1983method}, which was originally designed for accelerating strongly-concave (convex) optimization. Specifically, we propose the following proximal gradient ascent with Nesterov's momentum update for maximizing the strongly-concave part of the problem (P).

\textbf{(Nesterov's momentum):}
\begin{equation}
\left\{
\begin{aligned}
\widetilde{y}_t &= y_t + \gamma (y_t - y_{t-1}), \\
y_{t+1} &= \prox{\eta_y h}\big(\widetilde{y}_t + \eta_{y}\nabla_2 f({x_{t+1}}, \widetilde{y}_t)\big).
\end{aligned}
\right.
\end{equation}
To elaborate, the first step is a regular extrapolation step that involves momentum, which is the same as the first step of the heavy-ball scheme. The only difference from the heavy-ball scheme is that the starting point of the second proximal gradient step changes from $y_t$ to its extrapolated point $\widetilde{y}_t$.

We refer to the above algorithm design as \textbf{proximal-AltGDA with momentum (proximal-AltGDAm)}, and the algorithm updates are formally presented in \Cref{algo: prox-minimax}. It can be seen that proximal-AltGDAm is a simple algorithm that has a single loop structure, and adopts {alternating updates with} momentum acceleration. More importantly, it involves only standard hyper-parameters such as the learning rates and momentum parameters and therefore is easy to implement in practice.
Such a practical algorithm is much simper than the other accelerated GDA-type algorithms that involve complex nested-loop structure and require fine-tuned function smoothing \cite{Lin2020a,zhang2021complexity}. 

\begin{algorithm}
	\caption{Proximal Alternating GDA with Momentum (proximal-AltGDAm)}\label{algo: prox-minimax}
	\label{alg: 1}
	{\bf Input:}  Initialization $x_{-1}=x_0, y_{-1}=y_0$, learning rates $\eta_{x}, \eta_y$, momentum parameters $\beta,\gamma$. \\
	\For{ $t=0,1, 2,\ldots, T-1$}{
		\begin{align*}
		\widetilde{x}_t &= x_t + \beta (x_t - x_{t-1}),\\
		x_{t+1} &= \prox{\eta_x g}\big(\widetilde{x}_t - \eta_{x}\nabla_1 f(x_t,y_t)\big),\\
		\widetilde{y}_t &= y_t + \gamma (y_t - y_{t-1}), \\
		y_{t+1} &= \prox{\eta_y h}\big(\widetilde{y}_t + \eta_{y}\nabla_2 f(x_{t+1}, \widetilde{y}_t)\big).
		\end{align*}
	}
	\textbf{Output:} $x_T, y_T$.
\end{algorithm}


\section{Convergence and Computation Complexity of Proximal-AltGDAm}\label{sec:global}
Although the proposed proximal-AltGDAm algorithm applies the popular heavy-ball momentum and Nesterov's momentum to the GDA updates in a straightforward way, its convergence analysis cannot directly follow from the existing studies of momentum accelerated gradient descent algorithms. To explain more specifically, notice that in the $x$-proximal gradient update of \Cref{alg: 1}, it involves the partial gradient $\nabla_1 f(x_t, y_t)$, which corresponds to the gradient of the time-varying function $f(\cdot, y_t)$ (since $y_t$ changes over time $t$). Similarly, the $y$-proximal gradient update involves the gradient of another time-varying function $f(x_{t+1}, \cdot)$. Therefore, 
both momentum accelerated proximal gradient updates are actually applied to time-varying functions due to the nature of GDA updates. In this sense, the existing analysis of momentum accelerated gradient descent algorithms cannot be directly applied to analyze this algorithm.


To analyze the convergence of \Cref{alg: 1}, we first study the $x$-proximal gradient update with heavy-ball momentum and obtain the following characterization of per-iteration progress on the objective function value.

\begin{restatable}{proposition}{proplya}\label{prop: lyapunov}
	Let \Cref{assum: f} hold. 
Then, the parameter sequences $\{x_t,y_t\}_t$ generated by proximal-AltGDAm satisfy, for all $t=0,1,2,...,$
    \begin{align}
    \Phi&(x_{t+1}) +g(x_{t+1}) \nonumber\\
    &\le \Phi(x_{t})+g(x_t) - \Big(\frac{1-\beta}{2\eta_x} -2L\kappa^{\frac{11}{6}} \Big)\|x_{t+1}-x_t\|^2\nonumber\\
	&\quad +  \frac{\beta}{2\eta_x} \|x_t - x_{t-1}\|^2 + \frac{L}{4\kappa^{\frac{11}{6}}}\|y^*(x_t) - y_t\|^2. \label{eq: Hdec}
    \end{align}
\end{restatable}

The above proposition tracks the per-iteration optimization progress made by the $x$-proximal gradient update with heavy-ball momentum. To elaborate, the increment terms $\|x_{t+1} - x_t\|^2, \|x_{t} - x_{t-1}\|^2$ are induced by the heavy-ball momentum scheme. Moreover, since the $x$-update uses the partial gradient $\nabla_1 f(x_t, y_t)$ to approximate the exact gradient $\nabla \Phi(x_t)=\nabla_1 f(x_t, y^*(x_t))$, it naturally induces a tracking error term $\|y_t - y^*(x_t)\|^2$ that tracks the optimization gap of the $y$-update. Hence, we need to further bound this tracking error term by analyzing the $y$-proximal gradient update with Nesterov's momentum, which we explore next.

As explained earlier, the momentum accelerated $y$-updates in proximal-AltGDAm are applied to a time-varying strongly-concave function $f(x_{t+1},\cdot)$. Hence, the tracking error term $\|y_t - y^*(x_t)\|^2$ cannot be directly bounded using the standard convergence result of Nesterov's accelerated proximal gradient algorithm \cite{Nesterov2014}. Instead, we can view these $y$-updates as inexact accelerated proximal gradient updates. To elaborate, note that in the $t$-th iteration, the $y$-proximal gradient update is applied to the function $f(x_{t+1},\cdot)$. Then, we can rewrite the $y$-updates in all the previous iterations $k=0,1,...,t-1$ as follows.
\begin{align}
	y_{k+1} &= \prox{\eta_y h}\big(\widetilde{y}_k + \eta_{y}\nabla_2 f(x_{t+1}, \widetilde{y}_k) \nonumber\\
	&\qquad\quad + \eta_{y}\underbrace{[\nabla_2 f(x_{k+1}, \widetilde{y}_k) - \nabla_2 f(x_{t+1}, \widetilde{y}_k)]}_{\mathbf{e}_{k+1}}\big). \label{eq: approxPGA}
\end{align}
That is, we can view all the previous $y$-updates as applied to the fixed function $f(x_{t+1},\cdot)$ with an inexactness $\mathbf{e}_{k+1}$ involved in the computation of the partial gradient $\nabla_2 f(x_{t+1}, \widetilde{y}_k)$. In summary, the $y$-updates of proximal-AltGDAm can be understood as inexact accelerated gradient updates applied to the function $f(x_{t+1},\cdot)$ at time $t$. In particular, under the smoothness condition in \Cref{assum: f}, the inexactness is bounded as $\|\mathbf{e}_{k+1}\|\le L\|x_{k+1} - x_{t+1}\|$. Consequently, we can leverage the existing convergence result of inexact accelerated gradient algorithm \cite{schmidt2011convergence} to bound the optimality gap $\|y_{t}-y^*(x_{t})\|^2$ as follows.

\begin{restatable}{proposition}{propdy}\label{prop: dy}
	Let \Cref{assum: f} hold. Choose learning rate $\eta_y=\frac{1}{L}$ and momentum parameter $\gamma = \frac{\sqrt{\kappa}-1}{\sqrt{\kappa}+1}$. Then, the parameter sequences $\{x_t,y_t\}$ generated by proximal-AltGDAm satisfy, for all $t=0,1,2,...$
	\begin{align}
	\|y_{t+1}-&y^*(x_{t+1})\|^2 \le \frac{2R\kappa}{\sqrt{L}}(1-\kappa^{-0.5})^{t+1}  \nonumber\\
	&\quad+ \frac{6\kappa^2}{\sqrt{L}} \sum_{j=1}^{t} (1-\kappa^{-0.5})^{t+1-j/2} \|x_{j+1}-x_j\|.  \label{eq: H}
	\end{align}
\end{restatable}

Intuitively, in the above bound, the first term on the right hand side corresponds to the normal accelerated convergence rate, and the other term is induced by the inexactness $\mathbf{e}_k$. As both terms are scaled by the accelerated convergence factor $(1-\kappa^{-0.5})$, we expect that the above bound converges fast and further facilitates the convergence of \cref{eq: Hdec}. 
Next, substituting \cref{eq: H} into \cref{eq: Hdec} and telescoping, we obtain the following asymptotic stability properties of proximal-AltGDAm. 

\begin{restatable}{coro}{corolya}\label{coro: 1}
Under the conditions of \Cref{prop: dy} and stepsize $\eta_x\le 1/(16L\kappa^{\frac{11}{6}})$, the sequences $\{x_t,y_t\}_t$ generated by proximal-AltGDAm satisfy
\begin{align} 
	\|x_{t+1} - x_t \|, \|y_{t} - y^*(x_t) \|, \|y_{t+1} - y_t \|  \overset{t}{\to} 0. \nonumber
\end{align}
\end{restatable}

\textbf{Remark 1.} In \cite{chen2021}, the proximal-GDA algorithm (without alternating update and momentum) uses a small learning rate $\eta_{x} \le \mathcal{O}({L^{-2}}\kappa^{-3})$ to establish convergence. As a comparison, proximal-AltGDAm allows to choose a much larger learning rate $\eta_{x} \le {\mathcal{O}(L^{-1}\kappa^{-\frac{11}{6}})}$ to guarantee a faster convergence (proved later), thanks to the  momentum acceleration schemes. 


Therefore, if we can show that $\{x_t\}_t$ converges to a desired critical point $x^*$, then the above stability properties of proximal-AltGDAm implies that $\{y_t\}_t$ converges to the corresponding best response point $y^*(x^*)$.

To further characterize the global convergence property of proximal-AltGDAm, we define the following proximal gradient mapping associated with the objective function $\Phi(x)+g(x)$.
\begin{align}
G(x):=&\frac{1}{\eta_x}\Big(x-\prox{\eta_x g}\big(x - \eta_{x}\nabla \Phi(x)\big)\Big) \label{Phi_pg}.
\end{align}
The proximal gradient mapping is a standard metric for evaluating the optimality of nonconvex composite optimization problems \cite{Nesterov2014}. 
It can be shown that $x$ is a critical point of $(\Phi+g)(x)$ if and only if $G(x)=\zero$, and it reduces to the normal gradient when there is no regularizer $g$. Hence, our \textbf{convergence criterion} is to find a near-critical point $x$ that satisfies $\|G(x)\|\le \epsilon$ for some pre-determined accuracy $\epsilon >0$.

Next, by leveraging \Cref{prop: lyapunov} and \Cref{prop: dy}, we obtain the following global convergence result of proximal-AltGDAm and characterize its computational complexity (number of gradient and proximal mapping evaluations).

\begin{restatable}[Global convergence]{thm}{thmconv}\label{thm: 1}
Under the conditions of \Cref{prop: dy} and stepsize $\eta_x\le 1/(16L\kappa^{\frac{11}{6}})$, {the sequence $\{x_t\}_t$ generated by proximal-AltGDAm is bounded and has a compact set of limit points.  Also, every such limit point is a critical point of $(\Phi+g)(x)$}. Moreover, the total number of iterations $T$ required to achieve ${\min_{0\le t\le T}\|G(x_t)\|}\le\epsilon$ is $T=\mathcal{O}\big(\kappa^{\frac{11}{6}}\epsilon^{-2}\big)$, which is also the order of the required computational complexity.
\end{restatable}

To elaborate, the first statement of \Cref{thm: 1} shows that the sequence generated by proximal-AltGDAm converges to critical points of the minimax problem. The second statement proves that the {computation} complexity of proximal-AltGDAm for finding a near-critical point is of the order $\mathcal{O}\big(\kappa^{\frac{11}{6}}\epsilon^{-2}\big)$, which strictly improves the complexity $\mathcal{O}\big(\kappa^2\epsilon^{-2}\big)$ of both GDA \cite{lin2019gradient} and proximal-AltGDA \cite{boct2020alternating} in nonconvex-strongly concave minimax optimization. We note that the improvement is substantial when the problem condition number $\kappa = L/\mu$ is large, while the dependence on $\epsilon^{-2}$ is generally unimprovable in nonconvex optimization. 
In addition, thanks to the momentum acceleration schemes, our choice of learning rate $\eta_x=\mathcal{O}(L^{-1}\kappa^{-\frac{11}{6}})$ is more flexible than that $\eta_x=\mathcal{O}(L^{-1}\kappa^{-2})$ adopted by these GDA-type algorithms. These improvements are not only attributed to momentum acceleration, but also to the elaborate selection of the coefficients in the proof that aims to minimize the dependence of the complexity on $\kappa$.

\section{Experiments}
In this section, we compare the performance of proximal-AltGDAm with that of other GDA-type algorithms via numerical experiments. Specifically, we compare proximal-AltGDAm with the standard proximal-GDA/AltGDA algorithm \cite{chen2021} and the single-loop accelerated AltGDA algorithm (APDA) \cite{Zhu2020a}.  All these algorithms have a single-loop structure.

We consider solving the following regularized Wasserstein robustness model (WRM) \cite{Sinha2018CertifyingSD} using the MNIST dataset \cite{Lecun98}. 
\begin{align}
	\min_\theta \max_{\{\xi_i\}^n_{i=1}} &\frac{1}{n} \sum_{i=1}^{n} \Big[  \ell(h_\theta(\xi_i), y_i) - \lambda \| \xi_i - x_i\|^2 \Big] \nonumber\\
	&\qquad\qquad-\lambda_1\sum_{i=1}^n\|\xi_i\|_1 + \frac{\lambda_2}{2}\|\theta\|^2,
\end{align}
where $n = 60$k is the number of training samples, $\theta$ is the model parameter, $(x_i,y_i)$ corresponds to the $i$-th data sample and label, respectively, and $\xi_i$ is the adversarial sample corresponding to $x_i$. We choose the cross-entropy loss function $\ell$. We add an $\ell_1$ regularization to impose sparsity on the adversarial examples, and add an $\ell_2^2$ regularization to prevent divergence of the model parameters.

We set $\lambda=1.0$ that suffices to make the maximization part be strongly-concave, and set $\lambda_1=\lambda_2=10^{-4}$. We use a convolution network that consists of two convolution blocks followed by two fully connected layers. Specifically, each convolution block contains a convolution layer, a max-pooling layer with stride step $2$, and a ReLU activation layer. The convolution layers in the two blocks have $1,10$ input channels and $10,20$ output channels, respectively, and both of them have kernel size $5$, stride step $1$ and no padding. The two fully connected layers have input dimensions $320,50$ and output dimensions $50,10$, respectively. 

We implement all three algorithms using full gradients with the whole training set of $60$k images.
We choose the same learning rates $\eta_{x}=\eta_{y} = 10^{-3}$ for all algorithms. For proximal-AltGDAm, we choose momentum $\beta = 0.25$ and  $\gamma=0.75$. For APDA, we choose the fine-tuned $\eta =2 \eta_{x}$.   
As the function $\Phi(x)$ cannot be exactly evaluated, we run $100$ steps of stochastic gradient ascent updates with learning rate $0.1$ to maximize $f(x_t, y)-h(y)$ and obtain an approximated $y^*(x_t)$, which is used to estimate $\Phi(x)+g(x)$. 
\begin{figure}[H]
	\centering
	\includegraphics[width=0.48\linewidth]{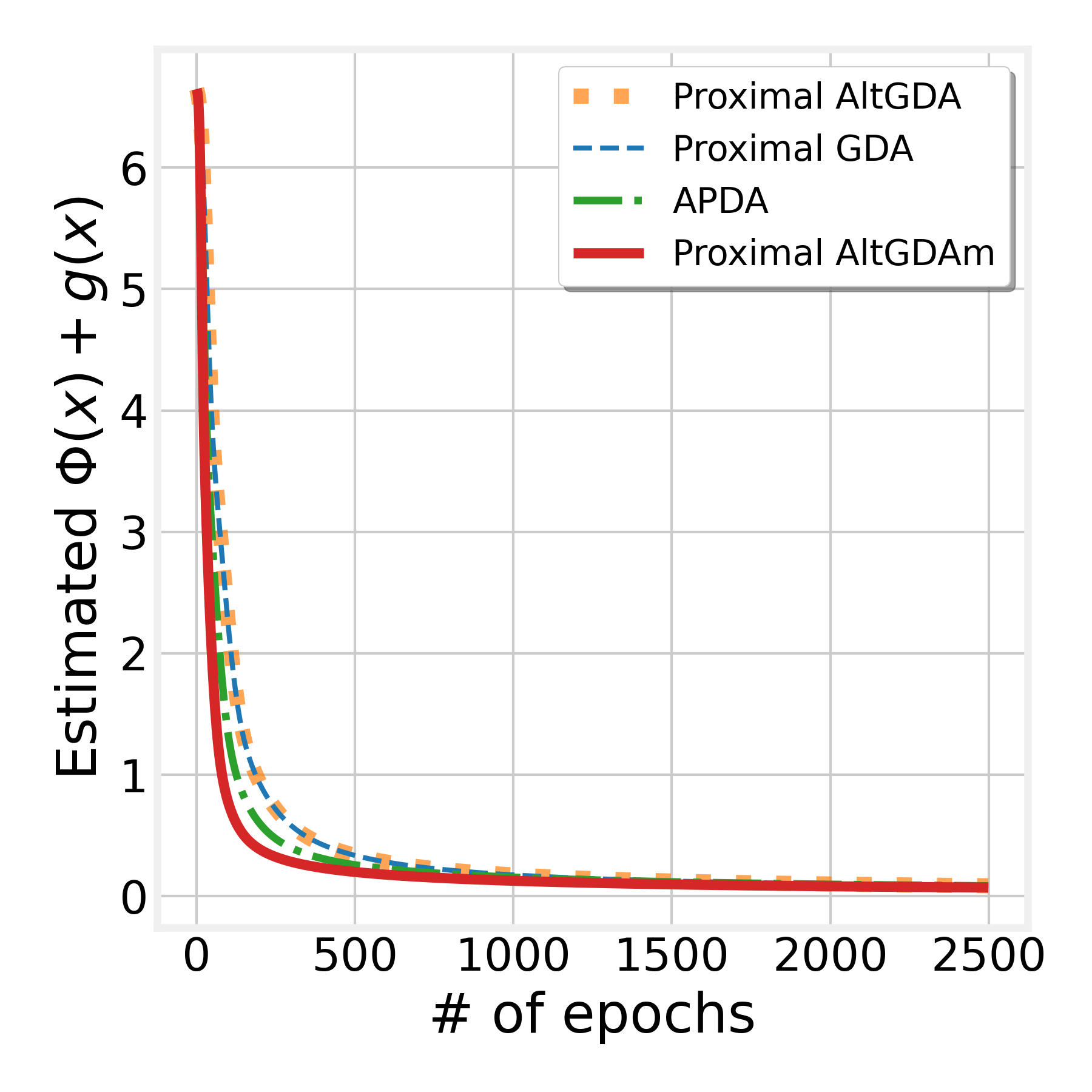}
	\includegraphics[width=0.48\linewidth]{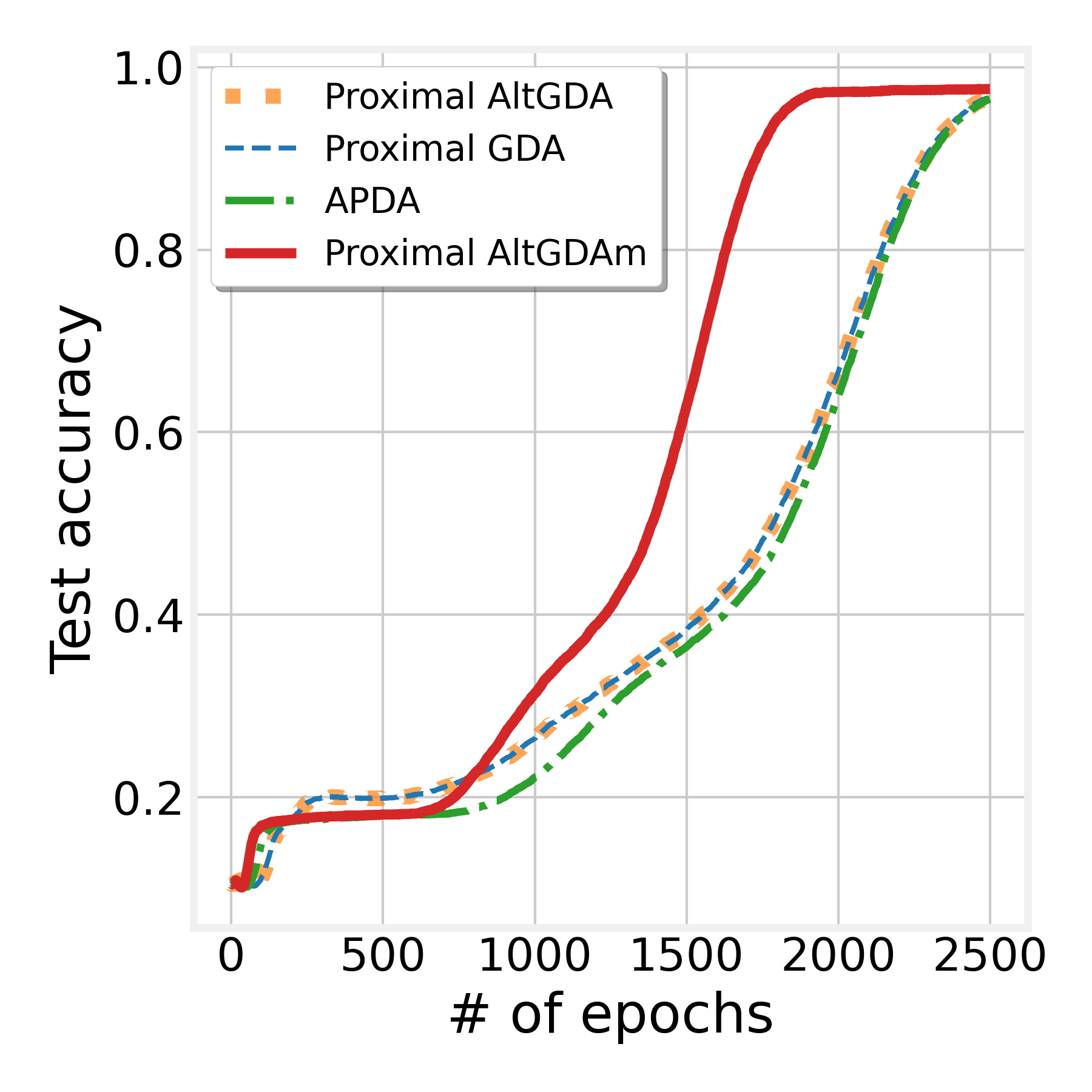}
	\vspace{-1mm}
	\caption{Left: comparison of $\Phi(x)+g(x)$ of all four algorithms. Right: comparison of the corresponding accuracy on the test dataset.}
	\label{fig:2}
	\vspace{-1mm}
\end{figure} 
\Cref{fig:2} (Left) compares the estimated objective function value achieved by all the three algorithms. It can be seen that proximal-AltGDAm achieves the fastest convergence among these algorithms and is significantly faster than the proximal-GDA and the proximal-AltGDA. This demonstrates the effectiveness of our simple momentum scheme.  \Cref{fig:2} (Right) further demonstrates the advantage of proximal-AltGDAm in the test accuracy. It can be seen that the robust model trained by proximal-AltGDAm achieves a much higher test accuracy.

\section{Conclusion}
We develop a single-loop and fast AltGDA algorithm that leverages proximal gradient updates and momentum acceleration to solve general regularized nonconvex-strongly-concave minimax optimization problems. By viewing the GDA updates of the algorithm as inexact accelerated gradient updates, we prove that the algorithm converges to a $\epsilon$-critical point with a computational complexity $\mathcal{O}(\kappa^{\frac{11}{6}}\epsilon^{-2})$, which substantially improves the state-of-the-art result. 
In the future work, it is interesting to develop a stochastic variant of this algorithm to further improve the sample complexity.


\section*{Acknowledgment}

The work of Ziyi Chen, Shaocong Ma and Yi Zhou was supported in part by U.S. National Science Foundation under the Grants CCF-2106216 and DMS-2134223.

\newpage
{
    \bibliographystyle{abbrv}
	\bibliography{./ref}
}

\newpage
\onecolumn
\appendices


\allowdisplaybreaks

\section{Proof of Proposition \ref{prop: lyapunov} } \label{sec_prop2proof}

\proplya*

\begin{proof}
Consider the $t$-th iteration of PGDAm. As the function $\Phi$ is $L(1+\kappa)$-smooth (from \Cref{prop_Phiystar}), we obtain that
\begin{align}
	\Phi(x_{t+1}) \le 	\Phi(x_{t}) + \inner{x_{t+1}-x_t}{\nabla \Phi(x_{t}) } + \frac{L(1+\kappa)}{2}\|x_{t+1}-x_t\|^2. \label{eq: 1}
\end{align}
On the other hand, by the definition of the proximal gradient step of $x_t$, we have that
\begin{align*}
	g(x_{t+1})+\frac{1}{2\eta_x} \|x_{t+1} - \widetilde{x}_t +\eta_{x} \nabla_1 f(x_t,y_t)\|^2 \le g(x_t)+\frac{1}{2\eta_x} \|x_t - \widetilde{x}_t + \eta_{x} \nabla_1 f(x_t,y_t)\|^2,
\end{align*}
which further simplifies to 
\begin{align}
g(x_{t+1}) &\le g(x_t) + \frac{1}{2\eta_x} \|x_{t} - \widetilde{x}_t\|^2 + \inner{x_{t} - \widetilde{x}_t }{\nabla_1 f(x_t,y_t)} \nonumber\\
&\quad- \frac{1}{2\eta_x} \|x_{t+1} - \widetilde{x}_t\|^2 - \inner{x_{t+1} - \widetilde{x}_t }{\nabla_1 f(x_t,y_t)} \nonumber\\
&\overset{(i)}{=} g(x_t) + \frac{\beta^2}{2\eta_x} \|x_{t} - {x}_{t-1}\|^2 - \frac{1}{2\eta_x} \|x_{t+1} - {x}_t - \beta (x_t - x_{t-1})\|^2  \nonumber\\
&\quad + \inner{x_t - x_{t+1}}{\nabla_1 f(x_t,y_t)} \nonumber\\
&= g(x_t) + \frac{\beta^2}{2\eta_x} \|x_{t} - {x}_{t-1}\|^2 - \frac{1}{2\eta_x} \|x_{t+1} - {x}_t\|^2 - \frac{\beta^2}{2\eta_x} \|x_t - x_{t-1}\|^2  \nonumber\\
&\quad + \frac{\beta}{\eta_x} \inner{x_{t+1} - x_t }{x_t - x_{t-1}}  +  \inner{x_t - x_{t+1}}{\nabla_1 f(x_t,y_t)} \nonumber\\
&\le g(x_t) - \frac{1}{2\eta_x} \|x_{t+1} - {x}_t\|^2  + \frac{\beta}{2\eta_x} \|x_{t+1} - x_t \|^2 +  \frac{\beta}{2\eta_x} \|x_t - x_{t-1}\|^2  +  \inner{x_t - x_{t+1}}{\nabla_1 f(x_t,y_t)}, \label{eq: 2}
\end{align}
where (i) uses the fact that $x_t - \widetilde{x}_t = \beta (x_{t-1} - x_t)$.

Adding up \cref{eq: 1} and \cref{eq: 2} yields that
\begin{align}
	&\Phi(x_{t+1}) +g(x_{t+1}) \nonumber\\
	&\le 	\Phi(x_{t})+g(x_t) - \Big(\frac{1}{2\eta_x} -\frac{L(1+\kappa)}{2} \Big)\|x_{t+1}-x_t\|^2 + \frac{\beta}{2\eta_x} \|x_{t+1} - x_t \|^2 +  \frac{\beta}{2\eta_x} \|x_t - x_{t-1}\|^2  \nonumber\\
	&\quad + \inner{x_{t+1}-x_t}{\nabla \Phi(x_{t}) - \nabla_1 f(x_t,y_t)} \nonumber\\
	&\le \Phi(x_{t})+g(x_t) - \Big(\frac{1}{2\eta_x} -\frac{L(1+\kappa)}{2} \Big)\|x_{t+1}-x_t\|^2 + \frac{\beta}{2\eta_x} \|x_{t+1} - x_t \|^2 +  \frac{\beta}{2\eta_x} \|x_t - x_{t-1}\|^2  \nonumber\\
	&\quad+ \|x_{t+1}-x_t\| \|\nabla \Phi(x_{t}) - \nabla_1 f(x_t,y_t)\| \nonumber\\
	&\le \Phi(x_{t})+g(x_t) - \Big(\frac{1}{2\eta_x} -\frac{L(1+\kappa)}{2} \Big)\|x_{t+1}-x_t\|^2 + \frac{\beta}{2\eta_x} \|x_{t+1} - x_t \|^2 +  \frac{\beta}{2\eta_x} \|x_t - x_{t-1}\|^2  \nonumber\\
	&\quad+ \|x_{t+1}-x_t\| \|\nabla_1 f(x_{t},y^*(x_t)) - \nabla_1 f(x_t,y_t)\| \nonumber\\
	&\le \Phi(x_{t})+g(x_t) - \Big(\frac{1}{2\eta_x} -L\kappa \Big)\|x_{t+1}-x_t\|^2 + \frac{\beta}{2\eta_x} \|x_{t+1} - x_t \|^2 +  \frac{\beta}{2\eta_x} \|x_t - x_{t-1}\|^2  \nonumber\\
	&\quad+ L\|x_{t+1}-x_t\| \|y^*(x_t) - y_t\| \nonumber\\
	&\stackrel{(i)}{\le} \Phi(x_{t})+g(x_t) - \Big(\frac{1-\beta}{2\eta_x} -L\kappa - L\kappa^{\frac{11}{6}} \Big)\|x_{t+1}-x_t\|^2 +  \frac{\beta}{2\eta_x} \|x_t - x_{t-1}\|^2 + \frac{L}{4\kappa^{\frac{11}{6}}}\|y^*(x_t) - y_t\|^2 \nonumber\\
	&\le \Phi(x_{t})+g(x_t) - \Big(\frac{1-\beta}{2\eta_x} - 2L\kappa^{\frac{11}{6}} \Big)\|x_{t+1}-x_t\|^2 +  \frac{\beta}{2\eta_x} \|x_t - x_{t-1}\|^2 + \frac{L}{4\kappa^{\frac{11}{6}}}\|y^*(x_t) - y_t\|^2.
\end{align}
where (i) uses AM-GM inequality that $ab\le \frac{Ca^2}{2}+\frac{b^2}{2C}$ for any $a,b,C\ge 0$. This proves eq. \eqref{eq: Hdec}


\end{proof}

\section{Proof of Proposition \ref{prop: dy} } \label{sec_prop3proof}

\propdy*

\begin{proof}
    We rewrite the inner accelerated gradient ascent steps in Algorithm \ref{algo: prox-minimax} as the inexact-proximal gradient method \eqref{eq: approxPGA}. 
    Then, based on Theorem 4 of \cite{schmidt2011convergence}, using $\eta_y=\frac{1}{L}$ and $\gamma=\frac{\sqrt{\kappa}-1}{\sqrt{\kappa}+1}$, this method has the following convergence rate.
    \begin{align}
        &f(x_{t+1},y_{t+1})-f(x_{t+1},y^*(x_{t+1})) \nonumber\\
        &\le (1-\kappa^{-0.5})^{t+1}\Big(\sqrt{2\big(f(x_{t+1},y_{t+1})-f(x_{t+1},y^*(x_{t+1}))\big)}+\sqrt{\frac{2}{\mu}}\sum_{i=1}^{t+1}\|e_i\|(1-\kappa^{-0.5})^{-i/2}\Big). \label{eq: inexact_rate}
    \end{align}
    The above convergence rate can be simplified as follows. 
	\begin{align}
		&\frac{\mu}{2}\|y_{t+1}-y^*(x_{t+1})\|^2 \nonumber\\
		&\overset{(i)}{\le} f(x_{t+1},y_{t+1})-f(x_{t+1},y^*(x_{t+1})) \nonumber\\
		&\le (1-\kappa^{-0.5})^{t+1}\Big(\sqrt{2\big(f(x_{t+1},y_{t+1})-f(x_{t+1},y^*(x_{t+1}))\big)}+\sqrt{\frac{2}{\mu}}\sum_{i=1}^{t+1}\|e_i\|(1-\kappa^{-0.5})^{-i/2}\Big)  \nonumber\\
		&\overset{(ii)}{\le} (1-\kappa^{-0.5})^{t+1}\sqrt{L\|y_{t+1}-y^*(x_{t+1})\|^2} + 
		\sqrt{\frac{2}{\mu}} \sum_{i=1}^{t+1}\|\nabla_2 f(x_i, \widetilde{y}_{i-1})-\nabla_2 f(x_{t+1}, \widetilde{y}_{i-1})\|(1-\kappa^{-0.5})^{t+1-i/2} \nonumber\\
		&\overset{(iii)}{\le} R\sqrt{L}(1-\kappa^{-0.5})^{t+1} +
		\sqrt{\frac{2}{\mu}} \sum_{i=1}^{t+1} (1-\kappa^{-0.5})^{t+1-i/2} \sum_{j=i}^{t} L\|x_{j+1}-x_j\| \nonumber\\
		&= R\sqrt{L}(1-\kappa^{-0.5})^{t+1} + 
		L\sqrt{\frac{2}{\mu}} \sum_{j=1}^{t} \sum_{i=1}^{j} (1-\kappa^{-0.5})^{t+1-i/2} \|x_{j+1}-x_j\| \nonumber\\
		&= R\sqrt{L}(1-\kappa^{-0.5})^{t+1} + 
		\sqrt{2L\kappa} \sum_{j=1}^{t} (1-\kappa^{-0.5})^{t+0.5} \frac{(1-\kappa^{-0.5})^{-j/2}-1}{(1-\kappa^{-0.5})^{-0.5}-1} \|x_{j+1}-x_j\| \nonumber\\
		&\le R\sqrt{L}(1-\kappa^{-0.5})^{t+1} + 
		\sqrt{2L\kappa} \sum_{j=1}^{t} \frac{(1-\kappa^{-0.5})^{t+1-j/2}}{1-(1-\kappa^{-0.5})^{0.5}} \|x_{j+1}-x_j\| \nonumber\\
		&\overset{(iv)}{\le} R\sqrt{L}(1-\kappa^{-0.5})^{t+1} + 
		2\kappa\sqrt{2L} \sum_{j=1}^{t} (1-\kappa^{-0.5})^{t+1-j/2} \|x_{j+1}-x_j\|, \nonumber
	\end{align}
	where (i) and (ii) use $\nabla_2 f(x_{t+1},y^*(x_{t+1}))=0$ and Assumption \ref{assum: f}.1 that $f(x,\cdot)$ is $L$-smooth and $\mu$-strongly concave, (iii) uses the fact that $\mathcal{Y}$ is bounded with diameter $R$ and Assumption \ref{assum: f}.1 that $f$ is $L$-smooth, and (iv) uses $\frac{1}{1-(1-\kappa^{-0.5})^{0.5}}=\frac{1+(1-\kappa^{-0.5})^{0.5}}{\kappa^{-0.5}}\le 2\kappa^{0.5}$. Multiplying the above inequality with $2/\mu$ proves \Cref{prop: dy}.
\end{proof}
\section{Proof of Corollary \ref{coro: 1}} \label{sec_coro1proof}
\corolya*

\begin{proof}
Telescoping eq. \eqref{eq: H} over $t=0, 1, \ldots, T-1$ yields that
\begin{align}
	&\sum_{t=0}^{T-1}\|y_{t+1}-y^*(x_{t+1})\|^2 \nonumber\\
	&\le \frac{2R\kappa}{\sqrt{L}}\sum_{t=0}^{T-1}(1-\kappa^{-0.5})^{t+1}  + \frac{6\kappa^2}{\sqrt{L}} \sum_{t=0}^{T-1}\sum_{j=1}^{t} (1-\kappa^{-0.5})^{t+1-j/2} \|x_{j+1}-x_j\| \nonumber\\
	&\le \frac{2R\kappa^{1.5}}{\sqrt{L}} + \frac{6\kappa^2}{\sqrt{L}}\sum_{j=1}^{T-1}\sum_{t=j}^{T-1} (1-\kappa^{-0.5})^{t+1-j/2} \|x_{j+1}-x_j\| \nonumber\\
	&\le \frac{2R\kappa^{1.5}}{\sqrt{L}} + \frac{6\kappa^{2.5}}{\sqrt{L}}\sum_{j=1}^{T-1} (1-\kappa^{-0.5})^{j/2} \|x_{j+1}-x_j\|\nonumber\\
	&\le \frac{2R\kappa^{1.5}}{\sqrt{L}} + \frac{3\kappa^{2.5}}{\sqrt{L}}\sum_{j=1}^{T-1} \Big( \frac{1}{\kappa^{\frac{7}{6}}\sqrt{L}}(1-\kappa^{-0.5})^j + \kappa^{\frac{7}{6}}\sqrt{L}\|x_{j+1}-x_j\|^2 \Big)\nonumber\\
	&\le \frac{2R\kappa^{1.5}}{\sqrt{L}} + \frac{3\kappa^{\frac{11}{6}}}{L} + 3\kappa^{\frac{11}{3}}\sum_{j=1}^{T-1}\|x_{j+1}-x_j\|^2. \label{eq: sumH}
	\end{align}
Then, telescoping eq. \eqref{eq: Hdec} over $t=0, 1, \ldots, T-1$ yields that
\begin{align}
    &\Phi(x_T)+g(x_T)-\Phi(x_0)-g(x_0) \nonumber\\
    &\le - \Big(\frac{1-\beta}{2\eta_x} -2L\kappa^{\frac{11}{6}} \Big)\sum_{t=0}^{T-1}\|x_{t+1}-x_t\|^2  +  \frac{\beta}{2\eta_x} \sum_{t=0}^{T-1}\|x_t - x_{t-1}\|^2 + \frac{L}{4\kappa^{\frac{11}{6}}}\sum_{t=0}^{T-1}\|y^*(x_t) - y_t\|^2 \nonumber\\
    &\overset{(i)}{\le} -\Big(\frac{1-2\beta}{2\eta_x} -2L\kappa^{\frac{11}{6}} \Big)\sum_{t=0}^{T-1}\|x_{t+1}-x_t\|^2  + \frac{L}{4\kappa^{\frac{11}{6}}} \Big(R^2+\frac{2R\kappa^{1.5}}{\sqrt{L}} + \frac{3\kappa^{\frac{11}{6}}}{L} + 3\kappa^{\frac{11}{3}}\sum_{j=1}^{T-1}\|x_{j+1}-x_j\|^2\Big)\nonumber\\
    &\le -\Big(\frac{1-2\beta}{2\eta_x} -3L\kappa^{\frac{11}{6}} \Big)\sum_{t=0}^{T-1}\|x_{t+1}-x_t\|^2 + \frac{LR^2}{4\kappa^{\frac{11}{6}}} + \frac{R\sqrt{L}}{2\kappa^{\frac{1}{3}}}+1 \label{eq: dPhi}
\end{align}
where (i) uses $x_{-1}=x_0$, $\|y^*(x_0)-y_0\|\le R$ and eq. \eqref{eq: sumH}. When $\eta_x\le \frac{1}{16L\kappa^{\frac{11}{6}}}$ and $\beta\le \frac{1}{4}$, rearranging the above inequality yields that
\begin{align}
    L\kappa^{\frac{11}{6}} \sum_{t=0}^{T-1}\|x_{t+1}-x_t\|^2 &\le \Phi(x_0)+g(x_0)-\inf_{x\in\mathbb{R}^m} \big(\Phi(x)+g(x)\big) + \frac{LR^2}{4\kappa^{\frac{11}{6}}} + \frac{R\sqrt{L}}{2\kappa^{\frac{1}{3}}}+1 < +\infty \label{eq: sumx}
\end{align}

Letting $T\to \infty$ in the above inequality yields that $\sum_{t=0}^{\infty}\|x_{t+1}-x_t\|^2<+\infty$, so $\|x_{t+1}-x_t\|\overset{t}{\to} 0$. Then, letting $T\to \infty$ in eq. \eqref{eq: sumH} yields that $\sum_{t=0}^{\infty}\|y_{t+1}-y^*(x_{t+1})\|^2\le \frac{2R\kappa^{1.5}}{\sqrt{L}} + \frac{3\kappa^{\frac{11}{6}}}{L} + 3\kappa^{\frac{11}{3}}\sum_{j=1}^{\infty}\|x_{j+1}-x_j\|^2 < +\infty$, so $\|y_t-y^*(x_t)\|\overset{t}{\to} 0$. 

The last term $\|y_{t+1} - y_t\|\overset{t}{\to} 0$ can be proved as follows. 
\begin{align}
\|y_{t+1} - y_t\| &\le 	\|y_{t+1} - y^*(x_{t+1})\| + \|y^*(x_t)-y_{t}\| + \|y^*(x_{t+1})-y^*(x_t)\| \nonumber\\
&\overset{(i)}{\le} \|y_{t+1} - y^*(x_{t+1})\| + \|y_{t} - y^*(x_t)\| + \kappa\|x_{t+1}-x_t\| \overset{t}{\to} 0, \nonumber
\end{align}
where (i) uses the fact that $y^*$ is $\kappa$-Lipschitz.
\end{proof}

\section{Proof of Theorem \ref{thm: 1}} \label{sec_thm1proof}
\thmconv*

\begin{proof}
We first prove the existence of the limit points of $\{x_t\}$. Note that in eq. \eqref{eq: dPhi}, $\frac{1-2\beta}{2\eta_x} -2L\kappa^{\frac{11}{6}}\ge 0$ since $\eta_x\le \frac{1}{16L\kappa^{\frac{11}{6}}}$ and $\beta\le\frac{1}{4}$ as specified in \Cref{prop: dy}. Hence, for all $T\ge 0$,
\begin{align}
    \Phi(x_T)+g(x_T)&\le \Phi(x_0)+g(x_0) + \frac{LR^2}{4\kappa^{\frac{11}{6}}} + \frac{R\sqrt{L}}{2\kappa^{\frac{1}{3}}}+1 < +\infty, \nonumber
\end{align}
which implies that $\{\Phi(x_t)+g(x_t)\}_t$ is upper bounded. Hence, based on Assumption \ref{assum: f}.2, the sequence $\{x_t\}_t$ is bounded and thus has a compact set of limit points. \\

Next, we prove that every limit point $x$ of $\{x_t\}_t$ is a critical point of $(\Phi+g)(x)$, i.e., $\zero\in\partial(\Phi+g)(x)$.
By the optimality condition of the proximal gradient update of $x_{t+1}$ we have
\begin{align}
\zero &\in \partial g(x_{t+1}) + \frac{1}{\eta_x} \big(x_{t+1} - \widetilde{x}_t + \eta_{x} \nabla_1 f(x_t, y_t)\big) \nonumber\\
&=  \partial g(x_{t+1}) + \frac{1}{\eta_x} \big(x_{t+1} - x_t - \beta (x_t - x_{t-1}) + \eta_{x} \nabla_1 f(x_t, y_t)\big), \nonumber
\end{align}
which implies that $-\frac{1}{\eta_x} \big(x_{t+1} - x_t - \beta (x_t - x_{t-1}) + \eta_{x} \nabla_1 f(x_t, y_t)\big)\in \partial g(x_{t+1})$ and thus by convexity of $g$ we have
\begin{align}
    g(x)\ge g(x_{t(j)+1})-\frac{1}{\eta_x} \big(x_{t+1} - x_t - \beta (x_t - x_{t-1}) + \eta_{x} \nabla_1 f(x_t, y_t)\big)^{\top}(x-x_{t(j)+1}); \forall x\in\mathbb{R}^m. \label{eq: gge}
\end{align}
As $x_{t(j)}\overset{j}{\to} x^*$ and $\|y_{t(j)} - y^*(x_{t(j)}) \| \overset{t}{\to} 0$, we have $y_{t(j)}\overset{t}{\to} y^*(x^*)$ due to continuity of $y^*(\cdot)$. Also note that the convex function $g$ is continuous (See Corollary 10.1.1 of \cite{rockafellar1970convex}). Hence, letting $j\to\infty$ in eq. \eqref{eq: gge} yields that
\begin{align}
    g(x)\ge g(x^*)-\nabla_1 f(x^*, y^*(x^*)^{\top}(x-x^*)=g(x^*)-\nabla\Phi(x^*)^{\top}(x-x^*); \forall x\in\mathbb{R}^m, \label{eq: gge_star}
\end{align}
which further implies that $-\nabla \Phi(x^*)\in\partial g(x^*)\Rightarrow \zero\in\partial (\Phi+g)(x^*)$. Hence, $x^*$ in a critical point of $(\Phi+g)(x)$.\\

Finally, we derive the non-asymptotic computational complexity to obtain $\min_{0\le t\le T} \|G(x_t)\|\le \epsilon$. Note that
\begin{align}
	\|G(x_{t+1})\|=&\frac{1}{\eta_x} \big\|x_{t+1}-\text{prox}_{\eta_x g}\big(x_{t+1}-\eta_x\nabla\Phi(x_{t+1})\big)\big\|\nonumber\\
	\overset{(i)}{\le}& \frac{1}{\eta_x} \big\|x_{t+1}-\widetilde{x}_t+\eta_x\big[\nabla_1 f(x_t,y_t)-\nabla f_1\big(x_{t+1},y^*(x_{t+1})\big)\big]\big\|\nonumber\\
	\le& \frac{1}{\eta_x} \big\|x_{t+1}-x_t-\beta(x_t-x_{t-1})\big\|+L\|x_{t+1}-x_t\|+L\|y^*(x_{t+1})-y^*(x_t)\|+L\|y^*(x_t)-y_t\|\nonumber\\
	\overset{(ii)}{\le}& \Big(\frac{1}{\eta_x}+L+L\kappa\Big)\|x_{t+1}-x_t\| + \frac{\beta}{\eta_x}\|x_t-x_{t-1}\|+L\|y^*(x_t)-y_t\|, \nonumber
\end{align}
where (i) uses $x_{t+1} \in \prox{\eta_x g}\big(\widetilde{x}_t - \eta_{x}\nabla_1 f(x_t,y_t)\big)$, $\nabla\Phi(x)=\nabla f_1\big(x,y^*(x)\big)$ (from Proposition \ref{prop_Phiystar}) and the non-expansiveness of proximal mapping, (ii) uses the property that $y^*$ is $\kappa$-Lipschitz continuous in Proposition \ref{prop_Phiystar}. Hence, 
\begin{align}
    &(T-1)\min_{0\le t\le T} \|G(x_t)\|^2\nonumber\\
	&\le (T-1)\min_{1\le t\le T-1} \|G(x_{t+1})\|^2\nonumber\\
	&\le \sum_{t=1}^{T-1}\|G(x_{t+1})\|^2\nonumber\\
	&\le \sum_{t=1}^{T-1} \Big[3\Big(\frac{1}{\eta_x}+L+L\kappa\Big)^2\|x_{t+1}-x_t\|^2 + \frac{3\beta^2}{\eta_x^2}\|x_t-x_{t-1}\|^2+3L^2\|y^*(x_t)-y_t\|^2\Big] \nonumber\\
	&\overset{(i)}{\le} 3(18L\kappa^{\frac{11}{6}})^2
	\sum_{t=0}^{T-1}\|x_{t+1}-x_t\|^2 + 27L^2\kappa^{\frac{11}{3}}
	\sum_{t=0}^{T-1}\|x_{t+1}-x_t\|^2 + 3L^2\sum_{t=0}^{T-1}\|y^*(x_t)-y_t\|^2 \nonumber\\
	&\overset{(ii)}{\le} 999L^2\kappa^{\frac{11}{3}} \sum_{t=0}^{T-1}\|x_{t+1}-x_t\|^2  + 3L^2\Big(\frac{2R\kappa^{1.5}}{\sqrt{L}} + \frac{3\kappa^{\frac{11}{6}}}{L} + 3\kappa^{\frac{11}{3}}\sum_{j=1}^{T-1}\|x_{j+1}-x_j\|^2\Big) + 3L^2\|y^*(x_0)-y_0\|^2, \nonumber\\
	&\overset{(iii)}{\le} \frac{1008L^2\kappa^{\frac{11}{3}}}{L\kappa^{\frac{11}{6}}} \Big(\Phi(x_0)+g(x_0)-\inf_{x\in\mathbb{R}^m} \big(\Phi(x)+g(x)\big) + \frac{LR^2}{4\kappa^{\frac{11}{6}}} + \frac{R\sqrt{L}}{2\kappa^{\frac{1}{3}}}+1\Big) + 6RL^{1.5}\kappa^{1.5}+9L\kappa^{\frac{11}{6}} + 3L^2R^2 \nonumber\\
	&=\mathcal{O}(\kappa^{\frac{11}{6}}). \nonumber
\end{align}
where (i) uses $\beta\le \frac{1}{4}$ and the maximum possible stepsize $\eta_x=\frac{1}{16L\kappa^{\frac{11}{6}}}$, (ii) uses eq. \eqref{eq: sumH}, and (iii) uses eq. \eqref{eq: sumx} and the fact that $\mathcal{Y}$ is bounded with diameter $R$. Based on the above inequality, 
when the number of iterations $T\ge \mathcal{O}(\kappa^{\frac{11}{6}}\epsilon^{-2})$, $\min_{0\le t\le T} \|G(x_t)\|\le \sqrt{\mathcal{O}(\kappa^{\frac{11}{6}})/(T-1)}\le \epsilon$. Since each iteration has $\mathcal{O}(1)$ number of gradient and proximal mapping evaluations, the order of computational complexity is also $\mathcal{O}(\kappa^{\frac{11}{6}}\epsilon^{-2})$. 
\end{proof}

\section{Deriviation of computational complexities in Table \ref{table1}} \label{supp: table}

In this section, we will derive some computational complexities in Table \ref{table1} that are not directly shown in their corresponding papers. Note that all these GDA-type algorithms in Table \ref{table1} are single-loop. Hence, the computational complexity (the number of gradient evaluations) has the order of $\mathcal{O}(T)$ where $T$ is the number of iterations. 

First, the papers in Table \ref{table1} use different convergence measures for computational complexity. Specifically, \cite{chen2021,huang2021efficient} and our work show computational complexity to achieve $\|G(x)\|\le\epsilon$ where the proximal gradient mapping $G$ is defined in \eqref{Phi_pg}.  \cite{lin2019gradient,boct2020alternating} use the measure $\min_t \text{dist}\big(\Phi(x_t)+\partial g(x_t),\mathbf{0}\big)\le\epsilon$ where $\Phi(x):=\max_{y\in\mathcal{Y}} f(x,y)-h(y)$, $\partial g$ denotes the partial gradient of $g$ and $\text{dist}\big(A,\mathbf{0}\big)$ denotes the distance between $\mathbf{0}$ and any set $A$. \cite{xu2020unified} has no regularizers $g$, $h$ and uses the convergence measure $\min_t\|\nabla f(x_t,y_t)\|\le\epsilon$ when $y\in\mathbb{R}^d$ is unconstrained, which does not necessarily yield the desired approximate critical point of $\Phi$.

\subsection{Derivation of complexity in \cite{chen2021}}
In \cite{chen2021}, Proposition 2 states that the Lyapunov function $H(z_t):= \Phi(x_t) + g(x_t) + {(1-\frac{1}{4\kappa^2})}\|y_t-y^*(x_t)\|^2$ where $z_t:=(x_t,y_t)$ are generated by GDA decreases at the following rate. 
\begin{align}
    H(z_{t+1}) \le H(z_t) - 2\|x_{t+1}-x_t\|^2 -\frac{1}{4\kappa^2}\big(\|y_{t+1}-y^*(x_{t+1})\|^2+\|y_{t}-y^*(x_{t})\|^2\big). \label{eq: Hdec_GDA}
\end{align}

Note that for GDA, the gradient mapping \eqref{Phi_pg} has the following norm bound
\begin{align}
	\|G(x_{t+1})\|=&\frac{1}{\eta_x} \big\|x_{t+1}-\text{prox}_{\eta_x g}\big(x_{t+1}-\eta_x\nabla\Phi(x_{t+1})\big)\big\|\nonumber\\
	\overset{(i)}{\le}& \frac{1}{\eta_x} \big\|x_{t+1}-x_t+\eta_x\big[\nabla_1 f(x_t,y_t)-\nabla f_1\big(x_{t+1},y^*(x_{t+1})\big)\big]\big\|\nonumber\\
	\le& \Big(\frac{1}{\eta_x}+L\Big)\|x_{t+1}-x_t\|+L\|y^*(x_{t+1})-y^*(x_t)\|+L\|y^*(x_t)-y_t\|\nonumber\\
	\overset{(ii)}{\le}& \Big(\frac{1}{\eta_x}+L+L\kappa\Big)\|x_{t+1}-x_t\|+L\|y^*(x_t)-y_t\|\, \nonumber
\end{align}
where (i) uses the GDA update rule $x_{t+1} \in \prox{\eta_x g}\big(x_t - \eta_{x}\nabla_1 f(x_t,y_t)\big)$, the expression $\nabla\Phi(x)=\nabla f_1\big(x,y^*(x)\big)$ in Proposition \ref{prop_Phiystar} and the non-expansiveness of proximal mapping, (ii) uses the property that $y^*$ is $\kappa$-Lipschitz continuous Proposition \ref{prop_Phiystar}. Hence, we obtain the following convergence rate. 
\begin{align}
    &\min_{0\le t\le T}\|G(x_t)\|^2 \le \frac{1}{T}\sum_{t=0}^{T-1}\|G(x_{t+1})\|^2 \nonumber\\
    &\le \frac{1}{T}\sum_{t=0}^{T-1}\Big[2\Big(\frac{1}{\eta_x}+L+L\kappa\Big)^2\|x_{t+1}-x_t\|^2 + 2L^2\|y^*(x_t)-y_t\|^2\Big] \nonumber\\
    &\stackrel{(i)}{\le} \frac{1}{T}\sum_{t=0}^{T-1}\Big[\mathcal{O}(\kappa^6)\|x_{t+1}-x_t\|^2 + 2L^2\|y^*(x_t)-y_t\|^2\Big] \nonumber\\
    &\le \frac{\mathcal{O}(\kappa^6)}{T} \sum_{t=0}^{T-1}\Big(2\|x_{t+1}-x_t\|^2+\frac{1}{4\kappa^2}\|y_{t+1}-y^*(x_{t+1})\|^2\Big) \nonumber\\
    &\stackrel{(ii)}{\le}\frac{\mathcal{O}(\kappa^6)}{T} \sum_{t=0}^{T-1}\big(H(z_t)-H(z_{t+1})\big)\nonumber\\
    &\le \frac{\mathcal{O}(\kappa^6)}{T}\big(H(z_0)-H(z_T)\big)
\end{align}
where (i) uses the maximum possible stepsize $\eta_x=\frac{1}{\kappa^{3}(L+3)^{2}}=\mathcal{O}(\kappa^{-3})$ in \cite{chen2021}. Therefore, To let $\min_{0\le t\le T}\|G(x_t)\|\le \epsilon$, the computational complexity has the order $T=\mathcal{O}(\kappa^6\epsilon^{-2})$.

\subsection{Derivation of complexity in \cite{huang2021efficient}}
The mirror descent ascent algorithm (Algorithm 1) in \cite{huang2021efficient} updates the variables $x$ and $y$ simultaneously using proximal mirror descent and momentum accelerated mirror ascent steps respectively. Specifically, using the Bregman functions $\psi_t(x):=\frac{1}{2}\|x\|^2$ and $\phi_t(y):=\frac{1}{2}\|y\|^2$ which are both $\rho=1$-strongly convex, this algorithm becomes proximal GDA with momentum on $y$ variable.

Substituting $\rho=1$ into Theorem 2 that provides convergence rate under deterministic minimax optimization (i.e., there are no stochastic samples in the objective function), we obtain the following hyperparameter choices $\eta=\mathcal{O}(1)$, $L=L_f(1+\kappa)=\mathcal{O}(L_f\kappa)$ \footnote{$L_f$ in \cite{huang2021efficient} has the same meaning as our $L$, the Lipschitz parameter of $\nabla f$.}, $\lambda=\mathcal{O}(L_f^{-1})$, $\gamma=\mathcal{O}\big[\min\big(L^{-1},\frac{\mu/L_f}{\kappa^2},\frac{\mu/L_f}{L_f^2}\big)\big]=\mathcal{O}\big[\min(L_f^{-1}\kappa^{-1},\kappa^{-3},L_f^{-2}\kappa^{-1})\big]=\mathcal{O}(\kappa^{-3})$ (Without loss of generality, we assume $\mu\le 1$ which implies that $\kappa=L_f/\mu \ge L_f$). Then the convergence rate (25) becomes 
\begin{align}
    \min_{1\le t\le T}\|G(x_t)\|&\le \frac{1}{T}\sum_{t=1}^T\|\mathcal{G}_t\| \le \mathcal{O}\Bigg(\frac{\sqrt{\widetilde{F}(x_1)-F^*}+\Delta_1}{\sqrt{T\gamma\rho}}\Bigg)\nonumber\\
    &\stackrel{(i)}{=} \mathcal{O}\Bigg(\frac{\sqrt{\Phi(x_1)+g(x_1)-\inf_{x}\big(\Phi(x)+g(x)\big)}+\|y_1-y^*(x_1)\|}{\sqrt{T\kappa^{-3}}}\Bigg)\nonumber\\
\end{align}
where (i) uses the notations in \cite{huang2021efficient} that $\mathcal{G}_t=G(x_t)$, $\widetilde{F}(x)=\Phi(x)+g(x)$, $\Delta_1=\|y_1-y^*(x_1)\|$ and the above hyperparameter choices. Hence, to achieve $\min_{1\le t\le T}\|\mathcal{G}_t\|\le \epsilon$, the required computation complexity is $T\ge \mathcal{O}\Big(\kappa^3\epsilon^{-2}\big(\Phi(x_1)+g(x_1)-\inf_{x}(\Phi(x)+g(x)) + \|y_1-y^*(x_1)\|^2\big)\Big)$. In Table \ref{table1}, we only keep the dependence of $T(\epsilon)$ on $\epsilon\approx 0$ and $\kappa\gg 1$, which yields $\mathcal{O}\big(\kappa^3\epsilon^{-2}\big)$. 

\subsection{Derivation of complexity in \cite{xu2020unified}}

\cite{xu2020unified} aims to solve the following minimax optimization
\begin{align}
	\min_{x\in \mathcal{X}}\max_{y\in \mathcal{Y}}~ f(x,y). \nonumber
\end{align} 
where $\mathcal{X}$ and $\mathcal{Y}$ are nonempty closed convex sets and $\mathcal{Y}$ is also compact. The following AltGDA algorithm with projection mappings $\mathcal{P}_{\mathcal{X}}$ and $\mathcal{P}_{\mathcal{Y}}$ is analyzed for nonconvex-strongly concave geometry where $f$ is $L$-smooth\footnote{In Assumption 2.1 of \cite{xu2020unified}, let all the Lipschitz smooth parameters $L_x=L_y=L_{12}=L_{21}:=L$ for simplicity.} and $f(\cdot,y)$ is $\mu$-strongly concave for any $y\in\mathcal{Y}$.
\begin{align}
\left\{ \begin{gathered}
    x_{k+1}=\mathcal{P}_{\mathcal{X}}\big(x_k-\eta^{-1}\nabla_x f(x_k,y_k)\big) \hfill \\
    y_{k+1}=\mathcal{P}_{\mathcal{Y}}\big(y_{k}+\rho \nabla_{y} f(x_{k+1}, y_{k})\big) \hfill \\ 
\end{gathered}  \right.
\end{align}

Using the largest possible stepsizes $\eta^{-1}=\mathcal{O}(L^{-1}\kappa^{-3})$, $\rho=\mathcal{O}(\mu L^{-2})=\mathcal{O}(L^{-1}\kappa^{-1})$ that satisfies eq. (3.18), the following key variables in Theorem 3.1 can be computed as follows.
\begin{align}
    d_1=&\mathcal{O}(L^{-1}\kappa^{-5}) \nonumber\\
    F_1-\overline{F}=&f(x_1,y_1)-\min_{x\in\mathcal{X},y\in\mathcal{Y}} f(x,y) +\mathcal{O}(L\kappa^3\sigma_y^2) \nonumber
\end{align}
where $\sigma_y$ is the diameter of the compact set $\mathcal{Y}$. Hence the number of iterations (also the order of the computation complexity) required to achieve $\|\nabla_x f(x_k,y_k)\|\le \epsilon$, $\frac{1}{\rho}\big\|y_k-\mathcal{P}_{\mathcal{Y}}\big(y_k+\rho\nabla_y f(x_k,y_k)\big)\big\|\le \epsilon$ (note that this does not necessarily yields approximate critical point of $\Phi(x):=\max_{y\in\mathcal{Y}}f(x,y)$) is
\begin{align}
    T(\epsilon)=\mathcal{O}\Big(\frac{F_1-\underline{F}}{d_1\epsilon^2}\Big)=\mathcal{O}\Big(\epsilon^{-2}L\kappa^5\big(f_1-\min_{x\in\mathcal{X},y\in\mathcal{Y}} f(x,y)+32L\kappa^3\sigma_y^2\big)\Big). \nonumber
\end{align}
In Table \ref{table1}, we only keep the dependence of $T(\epsilon)$ on $\epsilon\approx 0$ and $\kappa\gg 1$, which yields $\mathcal{O}\big(\kappa^5\epsilon^{-2}\big)$.

\end{document}